\documentclass[10pt,twocolumn,a4paper]{article}

\usepackage{times}
\usepackage{epsfig}
\usepackage{graphicx}
\usepackage{amsmath}
\usepackage{amssymb}
\usepackage{amsthm}

\theoremstyle{definition}
\newtheorem{lemma}{Lemma}

\theoremstyle{remark}
\newtheorem*{remark}{Remark}

\usepackage{bbold}      
\usepackage{mathtools}  

\DeclareMathOperator*{\argmin}{arg\,min}

\newcommand{\figref}[1]{Fig.~\ref{#1}}

\newcommand{\tableref}[1]{Table~\ref{#1}}

\newcommand{\cf}{\emph{cf.~}}
\newcommand{\ie}{\emph{ie.~}}
\newcommand{\eg}{\emph{eg.~}}

\usepackage[shortlabels]{enumitem}

\newcommand{\realnumbers}{\mathbb{R}}
\providecommand{\R}{\realnumbers}

\newcommand{\neighbour}[2]{\delta^{#2}(#1)}
\newcommand{\neighbourOut}[1]{\neighbour{#1}{+}}
\newcommand{\neighbourIn}[1]{\neighbour{#1}{-}}

\newcommand{\nodes}{V}
\newcommand{\vertices}{\nodes}
\newcommand{\edges}{E}
\newcommand{\contractedEdges}{\mathcal{E}}

\newcommand{\partitionGraph}{\mathcal{G}}
\newcommand{\components}{\mathcal{V}}
\newcommand{\arcs}{\mathcal{A}}

\newcommand{\cost}[1]{c_{#1}}
\newcommand{\branchvar}[1]{y_{#1}}

\usepackage[acronym,shortcuts,nonumberlist]{glossaries}
\glsdisablehyper  

\newacronym{ilp}{ILP}{integer linear program}
\newacronym{lp}{LP}{linear program}
\newacronym{mlt}{\textsc{MLTP}}{moral lineage tracing problem}
\newacronym{mcb}{\textsc{MCBP}}{minimum cost branching problem}
\newacronym{mcp}{\textsc{MCMCP}}{minimum cost multicut problem}
\newacronym{mcbmp}{\textsc{MCBMP}}{minimum cost bipartite matching problem}

\usepackage{booktabs}  
\usepackage{multirow}
\usepackage{algorithm}
\usepackage{algpseudocode}
\usepackage{float}
\newfloat{algorithm}{t}{lop}

\usepackage{tabularx}

\usepackage{tikz}
\usetikzlibrary{positioning}
\usetikzlibrary{calc}
\usetikzlibrary{arrows,decorations.markings}

\graphicspath{{gfx/images/}}
\usepackage[export]{adjustbox}

\newcommand*\samethanks[1][\value{footnote}]{\footnotemark[#1]~}

\usepackage[top=2.5cm, bottom=2.5cm, left=2cm, right=2cm]{geometry}
\usepackage{fancyhdr}
\usepackage{titlesec}
    \titleformat{\section}{\large\bfseries}{\thesection}{0.5em}{}
    \titleformat{\subsection}{\normalfont\bf}{\thesubsection}{0.5em}{}
    \titleformat{\subsubsection}{\normalfont\normalsize\it}{\thesubsubsection}{0.5em}{}
    \titleformat{\paragraph}[runin]{\normalfont\bfseries}{\theparagraph}{0.5em}{}
    \titleformat{\subparagraph}[runin]{\normalfont\normalsize\it}{\thesubparagraph}{0.5em}{}
\usepackage[font=small,labelfont=bf,labelsep=space]{caption}

\usepackage[pagebackref=true,breaklinks=true,letterpaper=true,colorlinks,bookmarks=false]{hyperref}

\begin{document}

\title{Efficient Algorithms for Moral Lineage Tracing}

\author{
  Markus Rempfler\textsuperscript{1}\thanks{Authors contributed equally.}~,
  Jan-Hendrik Lange\textsuperscript{2}\samethanks,
  Florian Jug\textsuperscript{3},
  Corinna Blasse\textsuperscript{3},
  Eugene W.\ Myers\textsuperscript{3},\\
  Bjoern H.\ Menze\textsuperscript{1}
  and Bjoern Andres\textsuperscript{2} \\[.25cm]
  \textsuperscript{1} \textit{Institute for Advanced Study \& Department of Informatics, Technical University of Munich} \\
  \textsuperscript{2} \textit{Max Planck Institute for Informatics, Saarbr\"ucken} \\
  \textsuperscript{3} \textit{Max Planck Institute of Molecular Cell Biology and Genetics, Dresden}
}

\date{}

\maketitle

\begin{abstract}
  \textit{
Lineage tracing, the joint segmentation and tracking of living cells as they move and divide in a sequence of light microscopy images, is a challenging task.  
Jug et al.~\cite{jug-2016} have proposed a mathematical abstraction of this task, the moral lineage tracing problem (MLTP), whose feasible solutions define both a segmentation of every image and a lineage forest of cells.  
Their branch-and-cut algorithm, however, is prone to many cuts and slow convergence for large instances.  
To address this problem, we make three contributions:
(i) we devise the first efficient primal feasible local search algorithms for the MLTP,
(ii) we improve the branch-and-cut algorithm by separating tighter cutting planes and by incorporating our primal algorithms,
(iii) we show in experiments that our algorithms find accurate solutions on the problem instances of Jug et al.\ and scale to larger instances, leveraging moral lineage tracing to practical significance.}
\end{abstract}

\section{Introduction}

Recent advances in microscopy have enabled biologists to
observe organisms on a cellular level with higher spatio-temporal
resolution than before~\cite{chen-2014,greenbaum-2012,tomer-2012}.
Analysis of such microscopy sequences is key to several open questions
in biology, including embryonic development of complex
organisms~\cite{keller-2010,keller-2008}, tissue
formation~\cite{guillot-2013} or the understanding of metastatic
behavior of tumor cells~\cite{zervantonakis-2012}.  
However, to get from a sequence of raw microscopy images to biologically or
clinically relevant quantities, such as cell motility, migration
patterns and differentiation schedules, robust methods for 
\emph{cell lineage tracing} are required and have therefore received
considerable attention
\cite{amat-2014,amat-2013,chenouard-2014,li-2008,maska-2014,meijering-2012}.

Cell lineage tracing is typically considered a two step problem: In
the first step, individual cells are detected and segmented in every
image. Then, in the second step, individual cells are tracked over
time and, in case of a cell division, linked to their ancestor cell, 
to finally arrive at the lineage forest of all cells (Fig.~\ref{fig:teaser}).
The tracking subproblem is complicated by cells that enter or leave the field of view, 
or low temporal resolution that allows large displacements or even multiple
consecutive divisions within one time step. In addition to this,
mistakes made in the first step, leading to over- or undersegmentation
of the cells, propagate into the resulting lineage forest and cause
spurious divisions or missing branches, respectively.  
The tracking subproblem is closely related to multi-target tracking
\cite{berclaz-2011,tang-2015,wang-2014,insafutdinov-2017,tang-2017} or reconstruction of tree-like
structures~\cite{funke-2012,rempfler-2015-media,rempfler-2016,turetken-2016,turetken-2011}. 
It has been cast in the form of different optimization problems \cite{jug-2014,kausler-2012,padfield-2011,schiegg-2015,schiegg-2013}
that can deal with some of the mentioned difficulties, e.g., by selecting from multiple segmentation hypotheses~\cite{schiegg-2015,schiegg-2013}.

\begin{figure}
\includegraphics[width=\columnwidth]{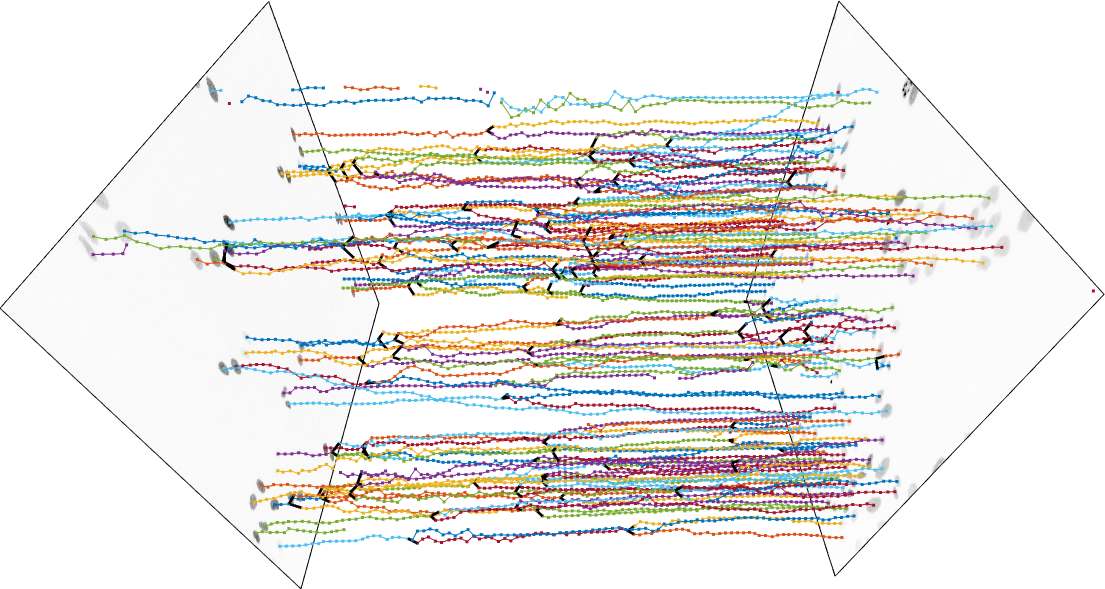}
\caption{Depicted above is a lineage forest of cells from a sequence of microscopy images. The first image of the sequence is shown on the left. The last image is shown on the right. Cell divisions are depicted in black.}
\label{fig:teaser}
\end{figure}

Jug et al.~\cite{jug-2016}, on the other hand, have proposed a rigorous mathematical abstraction of the joint problem which they call the \emph{\ac{mlt}}. 
It is a hybrid of the \emph{\ac{mcp}}, which
has been studied extensively for image
decomposition~\cite{andres-2011,andres-2012,andres-2013,bagon-2011,beier-2016,beier-2015,beier-2014,kappes-2016,keuper-2015,kim-2014,yarkony-2012,yarkony-2015},
and the \emph{minimum cost disjoint arborescence} problem, variations
of which have been applied to reconstruct lineage forests
in~\cite{jug-2014,kausler-2012,padfield-2011,schiegg-2013,schiegg-2015}
or tree-like structures~\cite{funke-2012,turetken-2011,turetken-2016}.
Feasible solutions to the \ac{mlt} define not only a valid cell
lineage forest over time, but also a segmentation of the cells in
every frame (\cf \figref{fig:lineage-forest}). Solving this
optimization problem therefore tackles both subtasks -- segmentation
and tracking -- simultaneously.  While Jug et al.~\cite{jug-2016}
demonstrate the advantages of their approach in terms of robustness,
they also observe that their branch-and-cut algorithm (as well as
the cutting-plane algorithm for the linear relaxation they study) is
prone to a large number of cuts and exhibits slow convergence on large
instances. That, unfortunately, prevents many applications of
the \ac{mlt} in practice, since it would be too computationally
expensive.

\vspace{-1ex}
\subparagraph{Contributions.} In this paper, we make three contributions:
Firstly, we devise two
efficient heuristics for the \ac{mlt}, both of which are primal
feasible local search algorithms inspired by the heuristics
of~\cite{keuper-2015,levinkov-2017} for the \ac{mcp}. We show that for fixed
intra-frame decompositions, the resulting subproblem can be solved
efficiently via bipartite matching.

Secondly, we improve the branch-and-cut algorithm \cite{jug-2016}
by separating tighter cutting planes and by employing our heuristics to
extract feasible solutions.

Finally, we demonstrate the convergence of our algorithms on the problem instances of~\cite{jug-2016}, solving two (previously unsolved) instances to optimality and obtaining accurate solutions orders of magnitude faster.
We demonstrate the scalability of our algorithms on larger (previously inaccessible) instances.

\section{Background and Preliminaries}

Consider a set of $\mathcal{T} = \{0, \dotsc, t_{\text{end}}\}$ consecutive frames of microscopy image data. In moral lineage tracing, we seek to jointly segment the frames into cells and track the latter and their descendants over time. This problem is formulated by \cite{jug-2016} as an \ac{ilp} with binary variables for all edges in an undirected graph as follows.

For each time index $t \in \mathcal{T}$, the node set $V_t$ comprises all cell fragments, \eg superpixels, in frame $t$. Each neighboring pair of cell fragments are connected by an edge. The collection of such edges is denoted by $E_t$. Between consecutive frames $t$ and $t+1$, cell fragments that are sufficiently close to each other are connected by a (temporal) edge. The set of such inter frame edges is denoted by $E_{t,t+1}$. By convention, we set $V_{t_{\text{end}}+1} = E_{t_{\text{end}}+1} = E_{t_{\text{end}},t_{\text{end}}+1} = \emptyset$. The graph $G = (V,E)$ with $V = \bigcup_{t \in \mathcal{T}} V_t $ and $E = \bigcup_{t \in \mathcal{T}} (E_t \cup E_{t,t+1})$ is called \emph{hypothesis graph} and illustrated in~\figref{fig:lineage-forest}. For convenience, we further write $G_t = (V_t,E_t)$ for the subgraph corresponding to frame $t$ and $G_t^+ = (V_t^+,E_t^+)$ with $V_t^+ = V_t \cup V_{t+1}$ and $E_t^+ = E_t \cup E_{t,t+1} \cup E_{t+1}$ for the subgraph corresponding to frames $t$ and $t+1$.

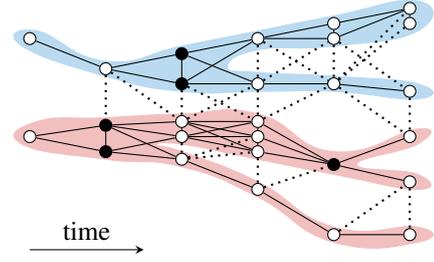
\begin{figure}
  \centering {
\definecolor{firstColor}{RGB}{20,133,204}
\definecolor{secondColor}{RGB}{212,42,42}

\tikzstyle{every node}=[circle, fill=white,
                        inner sep=0pt, minimum width=4pt]

\pgfdeclarelayer{bg}
\pgfsetlayers{bg,main}

\newcommand{\off}{0.3}
\newcommand{\offf}{0.2}

\tikzstyle{vertex}=[circle, draw=black, fill=white, inner sep=0pt, minimum width=1ex]
\tikzstyle{path}=[thick, dotted]

    \begin{tikzpicture}[yscale=-1]
    \draw[draw=firstColor!30, fill=firstColor!30] plot[smooth cycle, tension=0.9] coordinates
    	{(-0.1, 0.1) (1, 0.5) (2, 0.5 + \offf) (3, 0.5 + \offf) (4, 0.5 + \offf) (5.1, 0.6 + \offf) (5.1, 0.4 + \offf) 
    	(4, 0.3 + \offf) (3, 0.25 + \offf) (2.60, 0.28 + \offf/2) (3, 0.15) (4, 0.1) (5.1, -0.1) (5.1, -0.5)
    	(4, -0.3) (3, -0.1) (2, 0.1) (1, 0.3) (0.1, -0.1)};
    
    \draw[draw=secondColor!30, fill=secondColor!30] plot[smooth cycle, tension=0.9] coordinates
        {(-0.1, 1.1 + \off) (1, 1.3 + \off) (2, 1.4 + \off) (3, 1.8 + \off) (4, 2.4 + \off) (5.1, 2.4 + \off) (5.1, 2.2 + \off)
        (4, 2.2 + \off) (3, 1.6 + \off) (2, 1.2 + \off) (1.67, 1.12 + \off) (2, 1.1 + \off) (3, 1.3 + \off) (4, 1.5 + \off) (5.1, 1.7 + \off)
        (5.1, 1.5 + \off) (4.3, 1.37 + \off) (5.1, 1.1 + \off) (5.1, 0.9 + \off) (4, 1.2 + \off) (3, 0.7 + \off) (2, 0.7 + \off)
        (1, 0.75 + \off) (0.1, 0.8 + \off)};
	
    \node[style=vertex] at (0, 0) (a) {};
    \node[style=vertex] at (0, 1 + \off) (b) {};
    
    \node[style=vertex] at (1, 0.4) (c) {};
    \node[style=vertex, fill=black] at (1, 0.85 + \off) (d) {};
    \node[style=vertex, fill=black] at (1, 1.2 + \off) (e) {};
    
    
    \draw[color=black] (a) -- (c);
    \draw[color=black] (b) -- (d);
    \draw[color=black] (b) -- (e);
    
    \draw[style=path, color=black] (c) -- (d);
    \draw[color=black] (d) -- (e);
    
    \node[style=vertex, fill=black] at (2, 0.2) (a) {};
    \node[style=vertex, fill=black] at (2, 0.4 + \offf) (b) {};
    
    
    \node[style=vertex] at (2, 0.8 + \off) (f) {};
    \node[style=vertex] at (2, 1 + \off) (g) {};
    \node[style=vertex] at (2, 1.3 + \off) (h) {};
    
    \draw[color=black] (c) -- (a);
    \draw[color=black] (c) -- (b);
    \draw[style=path, color=black] (c) -- (f);
    \draw[color=black] (b) -- (a);
    \draw[style=path, color=black] (b) -- (f);
    
    \draw[color=black] (d) -- (f);
    \draw[color=black] (d) -- (g);
    \draw[color=black] (f) -- (g);
    \draw[color=black] (d) --  (h);
    \draw[color=black] (e) -- (h);
    \draw[color=black] (e) -- (g);
    \draw[style=path, color=black] (h) -- (g);
        
    \node[style=vertex] at (3, 0) (c) {};
    \node[style=vertex] at (3, 0.4 + \offf) (d) {};
    
    
    \node[style=vertex] at (3, 0.8 + \off) (z) {};
    \node[style=vertex] at (3, 1 + \off) (e) {};
    \node[style=vertex] at (3, 1.2 + \off) (i) {};
    \node[style=vertex] at (3, 1.7 + \off) (j) {};
    
    \draw[color=black] (a) -- (c);
    \draw[color=black] (b) -- (d);
    \draw[style=path, color=black] (c) -- (d);
    \draw[color=black] (a) -- (d);
    \draw[color=black] (b) -- (c);
    \draw[style=path, color=black] (b) -- (z);
    \draw[style=path, color=black] (d) -- (z);
    
    \draw[color=black] (f) -- (z);
    \draw[color=black] (g) -- (z);
    \draw[color=black] (z) -- (e);
    \draw[color=black] (f) -- (e);
    \draw[color=black] (g) -- (e);
    \draw[color=black] (f) -- (i);
    \draw[color=black] (g) -- (i);
    \draw[color=black] (e) -- (i);
    \draw[style=path, color=black] (f) -- (d);
    
    \draw[color=black] (h) -- (j);
    \draw[style=path, color=black] (h) -- (i);
    \draw[style=path, color=black] (h) -- (e);
    \draw[style=path, color=black] (i) -- (j);

    \node[style=vertex] at (4, -0.2) (a) {};
    \node[style=vertex] at (4, 0) (b) {};
    \node[style=vertex] at (4, 0.4 + \offf) (f) {};
    
    \node[style=vertex, fill=black] at (4, 1.37 + \off) (g) {};
    \node[style=vertex] at (4, 2.3 + \off) (k) {};
    
    \draw[color=black] (c) -- (a);
    \draw[color=black] (c) -- (b);
    \draw[color=black] (a) -- (b);
    \draw[style=path, color=black] (c) -- (f);
    \draw[style=path, color=black] (b) -- (f);
    \draw[color=black] (d) -- (f);
    \draw[style=path, color=black] (d) -- (b);
    \draw[style=path, color=black] (z) -- (f);
    
    \draw[color=black] (z) -- (g);
    \draw[color=black] (e) -- (g);
    \draw[color=black] (i) -- (g);
    
    \draw[color=black] (j) -- (k);
    \draw[style=path, color=black] (j) -- (g);
    
    \node[style=vertex] at (5, -0.4) (c) {};
    \node[style=vertex] at (5, -0.2) (d) {};
    \node[style=vertex] at (5, 0.5 + \offf) (e) {};
    
    \node[style=vertex] at (5, 1 + \off) (i) {};
    \node[style=vertex] at (5, 1.6 + \off) (j) {};
    \node[style=vertex] at (5, 2.3 + \off) (l) {};
	
	\draw[color=black] (a) -- (c);
	\draw[color=black] (b) -- (c);
	\draw[color=black] (c) -- (d);
	\draw[color=black] (f) -- (e);
	\draw[style=path, color=black] (f) -- (i);
	\draw[style=path, color=black] (f) -- (c);
	\draw[style=path, color=black] (f) -- (d);
	\draw[style=path, color=black] (b) -- (e);
	\draw[style=path, color=black] (e) -- (i);
	
	\draw[color=black] (g) -- (i);
	\draw[color=black] (g) -- (j);
	
	\draw[color=black] (k) -- (l);
	\draw[style=path, color=black] (j) -- (l);
	\draw[style=path, color=black] (k) -- (j);
	
    \draw[->,>=stealth,color=black] (0, 2.5 + \off) -- (1.5, 2.5 + \off);
    \node[label=center:{\color{black}time}] at (0.75, 2.25 + \off) {};

    \end{tikzpicture}

}
  \caption[\ac{mlt}]{The \emph{\acf{mlt}}\footnotemark[1]: Given a sequence of images decomposed
    into cell fragments (depicted as nodes in the figure), cluster
    fragments into cells in each frame and \emph{simultaneously}
    associate cells into lineage forests over time. Solid edges
    indicate joint cells within images and descendant relations across
    images. Black nodes depict fragments of cells about to divide.}
  \label{fig:lineage-forest}
\end{figure}

For any hypothesis graph $G = (V,E)$, a set $L \subseteq E$ is called a \emph{lineage cut} of $G$ and, correspondingly, the subgraph $(V,E \setminus L)$ is called a \emph{lineage (sub)graph} of $G$ if
\begin{enumerate}[1.,itemsep=-1ex]
\item For every $t \in \mathcal{T}$, the set $E_t \cap L$ is a {multicut}\footnotemark[2] of $G_t$.
\item For every $t \in \mathcal{T}$ and every $\{v,w\} \in E_{t,t+1} \cap L$, the nodes $v$ and $w$ are not path-connected in the graph $(V_t^+, E_t^+ \setminus L)$.
\item For every $t \in \mathcal{T}$ and nodes $v_t, w_t \in V_t$, $v_{t+1}, w_{t+1} \in V_{t+1}$ with $\{v_t,v_{t+1}\}, \{w_t,w_{t+1}\} \in E_{t,t+1} \setminus L$ and such that $v_{t+1}$ and $w_{t+1}$ are path-connected in $(V,E_{t+1} \setminus L)$, the nodes $v_t$ and $w_t$ are path-connected in $(V,E_t \setminus L)$.
\end{enumerate}
For any lineage graph $(V,E \setminus L)$ and every $t \in \mathcal{T}$, the non-empty, maximal connected subgraphs of $(V_t, E_t \setminus L)$ are called \emph{cells} at time index $t$. Furthermore, Jug et al.\ call a lineage cut, respectively lineage graph, \emph{binary} if it additionally satisfies
\begin{enumerate}
\setcounter{enumi}{3}
\item For every $t \in \mathcal{T}$, every cell at time $t$ is connected to at most two distinct cells at time $t+1$.
\end{enumerate}
According to \cite{jug-2016}, any lineage graph well-defines a lineage forest of cells. Moreover, a lineage cut (and thus a lineage graph) can be encoded as a 01-labeling on the edges of the hypothesis graph.

\footnotetext[1]{The figure is a correction of the one displayed in \cite{jug-2016}.}
\footnotetext[2]{A multicut of $G_t = (V_t,E_t)$ is a subset $M \subseteq E_t$ such that for every cycle $C$ in $G_t$ it holds that $\vert M \cap C \rvert \neq 1$, cf.\ \cite{hornakova-2017}.}

\begin{lemma}[\cite{jug-2016}]
For every hypothesis graph $G = (V,E)$ and every $x \in \{0,1\}^E$, the set $x^{-1}(1)$ of edges labeled $1$ is a lineage cut of $G$ iff $x$ satisfies inequalities \eqref{eq:space-cycle} -- \eqref{eq:morality}:
\begin{align}
& \forall t \in \mathcal{T} \forall C \in \text{cycles}(G_t) \forall e \in C : \nonumber \\ 
    & \quad x_e \leq \sum_{e' \in C \setminus \{e\}} x_{e'} \label{eq:space-cycle} \\
& \forall t \in \mathcal{T} \forall \{v,w\} \in E_{t,t+1} \forall P \in vw\text{-paths}(G_t^+): \nonumber \\ 
    & \quad x_{vw} \leq \sum_{e \in P} x_e \label{eq:space-time-cycle}
        \end{align}
    \begin{align}
& \forall t \in \mathcal{T} \forall \{v_t,v_{t+1}\}, \{w_t,w_{t+1}\} \in E_{t,t+1} (\text{with } v_t,w_t \in V_t) \nonumber \\ 
& \forall S \in v_tw_t\text{-cuts}(G_t) \forall P \in v_{t+1}w_{t+1}\text{-paths}(G_{t+1}) : \nonumber \\ 
    & \quad 1 - \sum_{e \in S} (1- x_e) \leq x_{v_tv_{t+1}} + x_{w_tw_{t+1}} + \sum_{e \in P} x_e \label{eq:morality}
\end{align}
\end{lemma}

Jug et al.\ refer to \eqref{eq:space-cycle} as \emph{space cycle}, to \eqref{eq:space-time-cycle} as \emph{space-time cycle} and to \eqref{eq:morality} as \emph{morality} constraints. 
We denote by $X'_G$ the set of all $x \in \{0,1\}^E$ that satisfy \eqref{eq:space-cycle} -- \eqref{eq:morality}. For the formulation of the additional \emph{bifurcation} constraints, which guarantee that the associated lineage cut is binary, we refer to \cite[Eq.\ 4]{jug-2016}. The set $X_G$ collects all $x \in X'_G$ that also satisfy the bifurcation constraints.

Given cut costs $c : E \to \R$ on the edges as well as \emph{birth} and \emph{termination} costs $c^+, c^- : V \to \R_0^+$ on the vertices of the hypothesis graph, \cite{jug-2016} defines the following \emph{moral lineage tracing problem} (\ac{mlt})
\begin{align}
\min_{x,x^+,x^-} \quad & \sum_{e \in E} c_ex_e + \sum_{v \in V} c_v^+ x_v^+ + \sum_{v \in V} c_v^- x_v^- \label{eq:mlt-objective}\\
\text{subject~to} \qquad & x \in X_G, \quad x^+, x^- \in \{0,1\}^V, \\
& \forall t \in \mathcal{T} \forall v \in V_{t+1} \forall S \in V_tv\text{-cuts}(G_t^+): \nonumber \\
& \qquad 1 - x_v^+ \leq \sum_{e \in S} (1 - x_e), \label{eq:birth} \\
& \forall t \in \mathcal{T} \forall v \in V_t \forall S \in vV_{t+1}\text{-cuts}(G_t^+): \nonumber \\
& \qquad 1 - x_v^- \leq \sum_{e \in S} (1 - x_e). \label{eq:termination}
\end{align}

The inequalities \eqref{eq:birth} and \eqref{eq:termination} are called \emph{birth} and \emph{termination} constraints, respectively.

\section{Local Search Algorithms}

In this section, we introduce two local search heuristics for the
\ac{mlt}. The first builds a lineage bottom-up in a greedy fashion,
while the second applies Kernighan-Lin~\cite{kernighan-1970} updates
to the intra-frame components. The latter requires repeatedly
optimizing a branching problem, given a fixed intra-frame
decomposition, for which we discuss an efficient combinatorial
minimizer.

Both algorithms maintain a decomposition of the graph
$(\nodes, \bigcup_{t \in \mathcal{T}} \edges_{t})$, \ie the components
within each frame $G_t$ that represent the cells. We denote the set of
all cells with $\components$. For each set of edges going from a
component $a \in \components$ at time point $t$ to a component $b$ at
$t+1$, we associate an arc $ab \in \arcs$. This gives a directed graph
$\partitionGraph=(\components, \arcs)$, as illustrated in
\figref{fig:partition-graph}. We write $\vertices_a$ for the set of
vertices $v$ in component $a \in \components$ and $\edges_{ab}$ for
the set of edges represented by arc $ab \in \arcs$. They further
maintain a selection of the arcs $\arcs(\branchvar{})$, where
$y \in \{0,1\}^\mathcal{A}$, to represent which temporal edges are
cut.

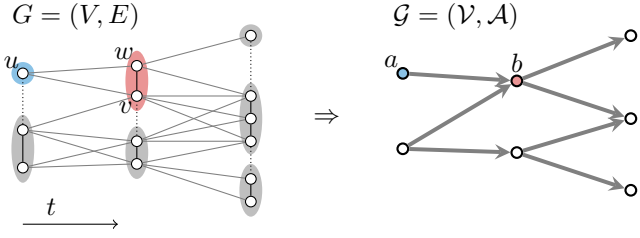
\begin{figure}
  \centering {

\definecolor{firstColor}{RGB}{20,133,204}
\definecolor{secondColor}{RGB}{212,42,42}

\tikzstyle{every node}=[circle, fill=white,
                        inner sep=0pt, minimum width=4pt]

\pgfdeclarelayer{bg}
\pgfsetlayers{bg,main}

\begin{tikzpicture}[thin]

  \node[rectangle, fill=none] at (0.7, 2.75) {$G=(V,E)$};
  \draw[->] (0,0) -- node[above, yshift=0.1cm, xshift=-.25cm] {$t$} (1.25, 0) ;
  
  \node[fill=none] at (4, 1.4) {$\Rightarrow$};
  \node[rectangle, fill=none] at (5.7, 2.75) {$\partitionGraph=(\components, \arcs)$};

  \node[draw] (v00) at (0, 0.75) {};
  \node[draw] (v01) at (0, 1.25) {};
  \node[draw,label={above left}:$u$] (v02) at (0, 2.) {};

  \node[draw] (v10) at (1.5, 0.8) {};
  \node[draw] (v11) at (1.5, 1.1) {};
  \node[draw,label={below left}:$v$] (v12) at (1.5, 1.7) {};
  \node[draw,label={above left}:$w$] (v13) at (1.5, 2.1) {};
  
  \node[draw] (v20) at (3, 2.5) {};
  \node[draw] (v21) at (3, 1.7) {};
  \node[draw] (v22) at (3, 1.4) {};
  \node[draw] (v23) at (3, 1.1) {};

  \node[draw] (v24) at (3, 0.6) {};
  \node[draw] (v25) at (3, 0.3) {};

  \draw (v00) -- (v01);
  \draw (v10) -- (v11);
  \draw (v12) -- (v13);
  \draw (v21) -- (v22);
  \draw (v22) -- (v23);
  \draw (v24) -- (v25);

  \begin{scope}[densely dotted]
    \draw (v01) -- (v02);
    \draw (v11) -- (v12);
    \draw (v20) -- (v21);
    \draw (v23) -- (v24);
  \end{scope}

    \begin{scope}[gray, thin]
      \draw (v00) -- (v10);
      \draw (v00) -- (v11);
      \draw (v01) -- (v10);
      \draw (v01) -- (v11);

      \draw (v01) -- (v12);
      \draw (v02) -- (v12);
      \draw (v02) -- (v13);

      \draw (v10) -- (v25);
      \draw (v10) -- (v24);
      \draw (v10) -- (v23);

      \draw (v11) -- (v23);
      \draw (v11) -- (v22);
      \draw (v11) -- (v21);
      
      \draw (v12) -- (v22);
      \draw (v12) -- (v21);
      \draw (v12) -- (v23);

      \draw (v13) -- (v20);
      \draw (v13) -- (v21);
    \end{scope}
  
    \begin{pgfonlayer}{bg}
      \fill[gray!50] ($(v00)!0.5!(v01)$) ellipse (0.15 and 0.45) {};
      \fill[firstColor!50] (v02) ellipse (0.15 and 0.15) {};
      
      \fill[gray!50] ($(v10)!0.5!(v11)$) ellipse (0.15 and 0.35) {};
      \fill[secondColor!50] ($(v12)!0.5!(v13)$) ellipse (0.15 and 0.4) {};

      \fill[gray!50] ($(v24)!0.5!(v25)$) ellipse (0.15 and 0.35) {};
      \fill[gray!50] ($(v21)!0.5!(v23)$) ellipse (0.15 and 0.475) {};
      
      \fill[gray!50] (v20) ellipse (0.15 and 0.15) {};
    \end{pgfonlayer}

    \begin{scope}[thick, xshift=5cm]
      \node[draw] (a00) at ($(v00)!0.5!(v01) + (5,0)$) {};
      \node[draw,label={above left}:$a$,fill=firstColor!50] (a01) at (0, 2.) {};

      \node[draw] (a10) at ($(v10)!0.5!(v11) + (5,0)$) {};
      \node[draw,label={above}:$b$,fill=secondColor!50] (a11) at ($(v12)!0.5!(v13) + (5,0)$) {};
      
      \node[draw] (a20) at (3, 2.5) {};
      \node[draw] (a21) at ($(v24)!0.5!(v25) + (5,0)$) {};
      \node[draw] (a22) at ($(v21)!0.5!(v23) + (5,0)$) {};
    \end{scope}

    \begin{scope}[gray, line width=1.5pt, >=stealth]
      \draw[->] (a00) -- (a10);
      \draw[->] (a00) -- (a11);
      \draw[->] (a01) -- (a11);
      
      \draw[->] (a11) -- (a20);
      \draw[->] (a11) -- (a22);
      \draw[->] (a10) -- (a22);
      \draw[->] (a10) -- (a21);
    \end{scope}
\end{tikzpicture}
}

  \caption{For a fixed decomposition of the frames (depicted with
    black solid/dashed cut edges), we associate a directed graph
    $\partitionGraph$ over the components $\components$. The arcs
    $\arcs$ bundle all edges going from any node of one cell to any
    node of another cell in the successive frame. For example, the
    components $\vertices_a=\{ u\}$ and $\vertices_b=\{ v,w \}$ are
    linked by the arc $ab$ which corresponds to the set of edges
    $E_{ab} = \{ uv, uw\}$.  Determining the optimal state of the
    temporal edges (grey) given a decomposition into cells boils down
    to finding an optimal branching in $\partitionGraph$.}
  \label{fig:partition-graph}
\end{figure}

\begin{algorithm}
\caption{Greedy Lineage Agglomeration (GLA)}
\begin{algorithmic}
\While{ $\mathrm{progress}$ }
   \State $(a,b) \gets \argmin_{ab \in \contractedEdges \cup \arcs} \Delta_{ab}^{\mathrm{transform}}$
   \If {$\Delta_{ab}^{\mathrm{transform}} < 0$}
       \State $\mathrm{applyTransform}(\mathcal{G}, a,b)$ \Comment \parbox[t]{.35\linewidth}{\raggedright updates partitions of $\partitionGraph$ and selects arcs $\arcs(\branchvar{})$.}
       \vspace{-1.8\baselineskip}
   \Else 
       \State \textbf{break}
   \EndIf
\EndWhile
\\
\Return $\mathrm{edgeLabels}(\mathcal{G})$ \Comment \parbox[t]{.4\linewidth}{\raggedright cut-edge labeling $x^*$ from $\components$ and $\arcs(\branchvar{})$.}
\end{algorithmic}
\label{alg:gla}
\end{algorithm}

\subsection{Greedy Lineage Agglomeration (GLA)}
\label{sec:gla}

The first algorithm takes an \ac{mlt} instance and constructs a
feasible lineage in a bottom-up fashion.  It is described in
Alg.~\ref{alg:gla} and follows a similar scheme as the
GAEC~\cite{keuper-2015} heuristic for the \ac{mcp} in the sense that
it always takes the currently best possible transformation, starting
from $\components = \vertices$. It applies three different types of
transformations: 1) a $\mathrm{merge}$ contracts all edges between two
components of the same time point $t$, combining them into one single
component.  2) $\mathrm{setParent}$ selects an arc $ab \in \arcs$ and
thereby sets $a$ of $\components_t$ as the (current) parent of $b \in
\components_{t+1}$, while 3) $\mathrm{changeParent}$ de-selects such
(active) arc $ab$ and instead selects an alternative $cb$. While final
components $\components$ determine intra-frame cuts, the final
selection of arcs then determines which temporal edges are cut edges
($x_e=1$). Unlike GAEC, transformations concerning the temporal edges
are reversible due to $\mathrm{changeParent}$. All allowed
transformations, $\mathrm{merge}$, $\mathrm{setParent}$ and
$\mathrm{changeParent}$, are depicted in \figref{fig:gla-moves}.
The change in objective \eqref{eq:mlt-objective} caused by a
particular transformation involving $a$ and $b$ is denoted with
$\Delta_{ab}^{\mathrm{transform}}$. In order to determine the cost or
reward of a particular transformation, we have to examine not only the
edge between the involved components $a$ and $b$, but also whether
they have an associated parent or child cell already. For a
$\mathrm{merge}$, we have to consider arcs going to children or
parents of either component, since they would be combined into an
active arc and therefore change their state and affect the
objective. The detailed, incremental calculation of these
transformation costs $\Delta_{ab}^{\mathrm{transform}}$ can be found
in the appendix. We maintain feasibility at all times: two components
with different parents cannot be merged (it would violate morality
constraints \eqref{eq:morality}), and similarly, a merge of two
partitions with a total of more than two active outgoing arcs is not
considered (as it would violate bifurcation constraints).  The
algorithm stops as soon as no available transformation decreases the
objective.

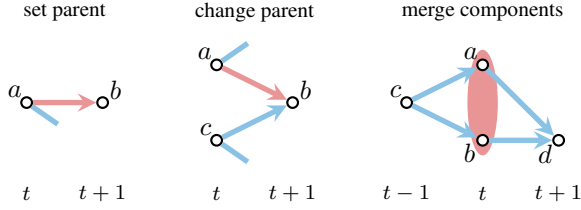
\begin{figure}
  \centering {
\definecolor{firstColor}{RGB}{20,133,204}
\definecolor{secondColor}{RGB}{212,42,42}

\tikzstyle{every node}=[circle, fill=white,
                        inner sep=0pt, minimum width=4pt]

\pgfdeclarelayer{bg}
\pgfsetlayers{bg,main}

\begin{tikzpicture}[thick]

  \node[draw,label={above left}:$a$] (a00) at (0, 0) {};
  \node[draw,label={above right}:$b$] (a01) at (1, 0.) {};

  \node[draw,label={above left}:$a$] (a10) at (2.5, 0.5) {};
  \node[draw,label={above right}:$b$] (a11) at (3.5, 0.) {};
  \node[draw,label={above left}:$c$] (a12) at (2.5, -0.5) {};

  \node[draw,label={above left}:$c$] (a20) at (5.0, 0.) {};
  \node[draw,label={above left}:$a$] (a21) at (6, 0.5) {};
  \node[draw,label={below left}:$b$] (a22) at (6, -0.5) {};
  \node[draw,label={below left}:$d$] (a23) at (7, -0.5) {};
  
  \begin{scope}[gray, line width=2pt, >=stealth]
    \draw[->, secondColor!50] (a00) -- (a01);
    \draw[firstColor!50] (a00) -- ++(-35:0.5cm);

    \draw[->, secondColor!50] (a10) -- (a11);
    \draw[->, firstColor!50] (a12) -- (a11);
    \draw[firstColor!50] (a12) -- ++(-35:0.5cm);
    \draw[firstColor!50] (a10) -- ++(35:0.5cm);
    
    \draw[->, firstColor!50] (a20) -- (a21);
    \draw[->, firstColor!50] (a20) -- (a22);
    \draw[->, firstColor!50] (a21) -- (a23);
    \draw[->, firstColor!50] (a22) -- (a23);
  \end{scope}
  
    \begin{pgfonlayer}{bg}
      \fill[secondColor!50] ($(a21)!0.5!(a22)$) ellipse (0.2 and 0.7) {};
    \end{pgfonlayer}
  
  \node[rectangle, black, fill=none] at (0.5, 1.2) {\footnotesize set parent};
  \node[rectangle, black, fill=none] at (3, 1.2) {\footnotesize change parent};
  \node[rectangle, black, fill=none] at (6, 1.2) {\footnotesize merge components};

  \node[black, fill=none] at (0, -1.2) {\footnotesize $t$};
  \node[black, fill=none] at (1, -1.2) {\footnotesize $t+1$};

  \node[black, fill=none] at (2.5, -1.2) {\footnotesize $t$};
  \node[black, fill=none] at (3.5, -1.2) {\footnotesize $t+1$};

  \node[black, fill=none] at (5, -1.2) {\footnotesize $t-1$};
  \node[black, fill=none] at (6, -1.2) {\footnotesize $t$};
  \node[black, fill=none] at (7, -1.2) {\footnotesize $t+1$};
\end{tikzpicture}
}
  \caption{The three transformations of GLA: set $a$ as parent of $b$
    (\textbf{left}), change the parent of $b$ from $c$ to $a$
    (\textbf{middle}) or merge two components $a$ and $b$ into one
    (\textbf{right}). The major arc along which the transformation
    occurs is depicted in red, while other arcs that affect the
    transformations cost are blue. When changing a parent,
    for example, the presence of other active arcs originating from
    $a$ and $c$ determine whether termination costs have to be
    paid. For a merge, we have to consider arcs to parents or
    children, which would be joined with an active arc and therefore
    change their state. }
  \label{fig:gla-moves}
\end{figure}

\vspace{-2ex}

\paragraph{Implementation.} We use a priority queue to efficiently
retrieve the currently best transformation. After applying it, each
\emph{affected} transformation is re-calculated and inserted into the
queue. We invalidate previous editions of transformations indirectly
by keeping track of the most recent version for all
$\contractedEdges$. For each component, we actively maintain the
number of children and its parent to represent the selected arcs
$\arcs(\branchvar{})$.

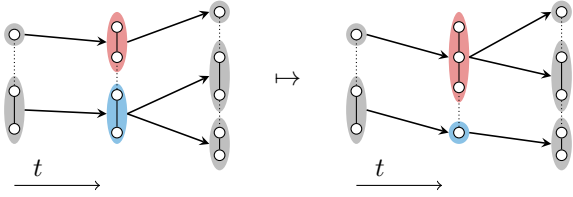
\begin{figure}
  \centering {

\definecolor{firstColor}{RGB}{20,133,204}
\definecolor{secondColor}{RGB}{212,42,42}

\tikzstyle{every node}=[circle, fill=white,
                        inner sep=0pt, minimum width=4pt]

\pgfdeclarelayer{bg}
\pgfsetlayers{bg,main}

\begin{tikzpicture}[xscale=0.9,thin]

  \draw[->] (0,0) -- node[above, yshift=0.1cm, xshift=-.25cm] {$t$} (1.25, 0) ;
 
  \node[draw] (v00) at (0, 0.75) {};
  \node[draw] (v01) at (0, 1.25) {};
  \node[draw] (v02) at (0, 2.) {};

  \node[draw] (v10) at (1.5, 0.7) {};
  \node[draw] (v11) at (1.5, 1.2) {};
  \node[draw] (v12) at (1.5, 1.7) {};
  \node[draw] (v13) at (1.5, 2.1) {};
  
  \node[draw] (v20) at (3, 2.3) {};
  \node[draw] (v21) at (3, 1.7) {};
  \node[draw] (v23) at (3, 1.2) {};

  \node[draw] (v24) at (3, 0.7) {};
  \node[draw] (v25) at (3, 0.4) {};

  \draw (v00) -- (v01);
  \draw (v10) -- (v11);
  \draw (v12) -- (v13);
  \draw (v21) -- (v23);
  \draw (v24) -- (v25);

  \begin{scope}[densely dotted]
    \draw (v01) -- (v02);
    \draw (v11) -- (v12);
    \draw (v20) -- (v21);
    \draw (v23) -- (v24);
  \end{scope}

    \begin{pgfonlayer}{bg}
      \fill[gray!50] ($(v00)!0.5!(v01)$) ellipse (0.15 and 0.45) {};
      \fill[gray!50] (v02) ellipse (0.15 and 0.15) {};
      
      \fill[firstColor!50] ($(v10)!0.5!(v11)$) ellipse (0.15 and 0.4) {};
      \fill[secondColor!50] ($(v12)!0.5!(v13)$) ellipse (0.15 and 0.4) {};

      \fill[gray!50] ($(v24)!0.5!(v25)$) ellipse (0.15 and 0.35) {};
      \fill[gray!50] ($(v21)!0.5!(v23)$) ellipse (0.15 and 0.475) {};
      
      \fill[gray!50] (v20) ellipse (0.15 and 0.15) {};
    \end{pgfonlayer}
    
     \begin{scope}[black, line width=0.6pt, >=stealth]
      \draw[->] ($(v00)!0.5!(v01) + (0.15,0)$) -- ($(v10)!0.5!(v11) - (0.15,0)$);
      \draw[->] ($(v02) + (0.15,0)$) -- ($(v12)!0.5!(v13) - (0.15,0)$);
      
      \draw[->] ($(v10)!0.5!(v11) + (0.15,0)$) -- ($(v24)!0.5!(v25) - (0.15,0)$);
      \draw[->] ($(v10)!0.5!(v11) + (0.15,0)$) -- ($(v21)!0.5!(v23) - (0.15,0)$);
      \draw[->] ($(v12)!0.5!(v13) + (0.15,0)$) -- ($(v20) - (0.15,0)$);
    \end{scope}
    
    \draw[->] (5,0) -- node[above, yshift=0.1cm, xshift=-.25cm] {$t$} (6.25, 0) ;
    \node[fill=none] at (4, 1.4) {$\mapsto$};
    
  \node[draw] (v00) at (5, 0.75) {};
  \node[draw] (v01) at (5, 1.25) {};
  \node[draw] (v02) at (5, 2.) {};

  \node[draw] (v10) at (6.5, 0.7) {};
  \node[draw] (v11) at (6.5, 1.3) {};
  \node[draw] (v12) at (6.5, 1.7) {};
  \node[draw] (v13) at (6.5, 2.1) {};
  
  \node[draw] (v20) at (8, 2.3) {};
  \node[draw] (v21) at (8, 1.7) {};
  \node[draw] (v23) at (8, 1.2) {};

  \node[draw] (v24) at (8, 0.7) {};
  \node[draw] (v25) at (8, 0.4) {};

  \draw (v00) -- (v01);
  \draw (v11) -- (v12);
  \draw (v12) -- (v13);
  \draw (v21) -- (v23);
  \draw (v24) -- (v25);

  \begin{scope}[densely dotted]
    \draw (v01) -- (v02);
    \draw (v10) -- (v11);
    \draw (v20) -- (v21);
    \draw (v23) -- (v24);
  \end{scope}

    \begin{pgfonlayer}{bg}
      \fill[gray!50] ($(v00)!0.5!(v01)$) ellipse (0.15 and 0.45) {};
      \fill[gray!50] (v02) ellipse (0.15 and 0.15) {};
      
      \fill[firstColor!50] ($(v10)$) ellipse (0.15 and 0.15) {};
      \fill[secondColor!50] ($(v11)!0.5!(v13)$) ellipse (0.15 and 0.6) {};

      \fill[gray!50] ($(v24)!0.5!(v25)$) ellipse (0.15 and 0.35) {};
      \fill[gray!50] ($(v21)!0.5!(v23)$) ellipse (0.15 and 0.475) {};
      
      \fill[gray!50] (v20) ellipse (0.15 and 0.15) {};
    \end{pgfonlayer}
    
     \begin{scope}[black, line width=0.6pt, >=stealth]
      \draw[->] ($(v00)!0.5!(v01) + (0.15,0)$) -- ($(v10) - (0.15,0)$);
      \draw[->] ($(v02) + (0.15,0)$) -- ($(v11)!0.5!(v13) - (0.15,0)$);
      
      \draw[->] ($(v10) + (0.15,0)$) -- ($(v24)!0.5!(v25) - (0.15,0)$);
      \draw[->] ($(v11)!0.5!(v13) + (0.15,0)$) -- ($(v21)!0.5!(v23) - (0.15,0)$);
      \draw[->] ($(v11)!0.5!(v13) + (0.15,0)$) -- ($(v20) - (0.15,0)$);
    \end{scope}
    
\end{tikzpicture}
}

  \caption{Depicted above is a transformation carried out by the KLB
    algorithm.  One node in the middle image is moved from the blue
    component to the red component.  Consequently, the optimal
    branching changes.}
  \label{fig:KLB-step}
\end{figure}

\begin{algorithm}
\caption{KL with Optimal Branchings (KLB)}
\begin{algorithmic}
\While{ progress }
   \For{ $a, b \in \components $ }
       \If{ $\not\exists uv \in \edges_t : u \in \vertices_a \wedge v \in \vertices_b $}
       \State \textbf{continue}
       \EndIf
       \State $\mathrm{improveBipartition}(\partitionGraph, a, b)$ 
       \Comment \parbox[t]{.24\linewidth}{\raggedright move nodes across border or merge.}
       \vspace{-1.3\baselineskip}
   \EndFor
   \For{ $a \in \components $ }
       \State $\mathrm{splitPartition}(\partitionGraph, a)$ \Comment split partition.
   \EndFor
\EndWhile
\\
\Return $\mathrm{cutEdgeLabels}(\partitionGraph)$ \Comment \parbox[t]{.4\linewidth}{\raggedright cut-edge labeling $x^*$ from $\components$ and $\arcs(\branchvar{}^*)$.}
\end{algorithmic}
\label{alg:klb}
\end{algorithm}

\subsection{Kernighan-Lin with Optimal Branchings (KLB)}
\label{sec:klb}
Algorithm~\ref{alg:klb} takes an \ac{mlt} instance and an initial
decomposition, \eg the result of GLA, and attempts to decrease the
objective function \eqref{eq:mlt-objective} in each step by changing
the intra-frame partitions in a
Kernighan-Lin-fashion~\cite{kernighan-1970}, an example is illustrated in Fig. \ref{fig:KLB-step}. Like the algorithm
proposed by~\cite{keuper-2015} for the related \ac{mcp}, it explores
three different local transformations to decrease the objective
function maximally: \textbf{a)} apply a sequence of $k$ node switches
between two adjacent components $a$ and $b$, \textbf{b)} a complete
merge of two components, and \textbf{c)} splitting a component into
two. Transforms that do not decrease the objective will be
discarded. In contrast to the setting of a \ac{mcp}, judging the
effect of such local modifications on the objective is more difficult,
since it requires according changes to the temporal cut-edges. This
can be seen when reordering the terms of the \ac{mlt} objective
$f_{\mathrm{MLTP}}$ \eqref{eq:mlt-objective}:
\begin{equation}
  \label{eq:composed-objective}
  f_{\mathrm{MLTP}}(x) = \sum_{\mathclap{e \in \bigcup_{t \in \mathcal{T}} \edges_{t,t+1}}} c_e \enspace + \quad \sum_{\mathclap{e \in \bigcup_{t \in \mathcal{T}} \edges_t}} c_e x_e + f_{\mathrm{MCBP}}(x)  \enspace ,
\end{equation}
where we identify the first sum to be an instance-dependent constant,
the second sum is the contribution from intra-frame edges (\ie the
decomposition into cells) and the last term, summarized with
$f_{\mathrm{MCBP}}$ is the sum over all inter-frame edges as well as
birth and termination costs. Given a particular KLB-transformation,
the change to the intra-frame part is straight-forward to calculate,
while the change of the inter-frame part involves solving $\min
f_{\mathrm{MCBP}}(.)$ anew.
This sub-problem turns out to be a variant of a \ac{mcb}, which we
discuss next. Afterwards, we describe a combinatorial optimizer for
this \ac{mcb}, and finally provide additional details on its usage
within KLB.

\paragraph{Minimum Cost Branching on $\partitionGraph$.}
Given a fixed decomposition into cells $\components$, \ie is a fixed
value for all intra-frame cut-edge variables $x_e$, we can reduce the
remaining (partial) \ac{mlt} to the following \ac{mcb} over
$\partitionGraph=(\components, \arcs)$:
\begin{align}
  \label{eq:mcb-objective}
  \min_{\branchvar{}, \branchvar{}^-, \branchvar{}^+} \quad & \sum_{ab\in \arcs} c_{ab} \branchvar{ab} + \sum_{a \in \components}c_a^+\branchvar{a}^+ + \sum_{a \in \components} c_a^- \branchvar{a}^-\\
  \mathrm{subject~to} \quad & \forall a \in \components \enspace :\enspace (1 - \branchvar{a}^+) = \sum_{\mathclap{b \in\neighbourIn{a}}} \branchvar{ba} \label{eq:incoming}\\
  & \forall a \in \components \enspace : \enspace (1 - \branchvar{a}^-) \leq \sum_{\mathclap{b \in \neighbourOut{a}}} \branchvar{ab} \leq 2 \label{eq:outgoing}\\
  & \branchvar{} \in \{ 0, 1\}^{\arcs}, \quad \branchvar{}^-,
  \branchvar{}^+ \in \{ 0, 1\}^{\components } \enspace ,
\end{align}
where $y,y^-,y^+$ are substitutes for those original cut variables
$x,x^+,x^-$ that are bundled within an arc or component in
$\partitionGraph$. The objective \eqref{eq:mcb-objective} is exactly
$f_{\mathrm{MCBP}}$ of \eqref{eq:composed-objective}. Each
$\branchvar{ab}$ indicates whether arc $ab$ is active
($\branchvar{ab}=1$) or not ($\branchvar{ab}=0$).  The equality
constraint \eqref{eq:incoming} ensures that at most one incoming arc
is selected (preventing a violation of morality) and, if none is
chosen, the birth indicator $ \branchvar{a}^+$ is active. In the same
sense, \eqref{eq:outgoing} enforces the penalty for termination if
necessary, and its upper bound limits the number of children to 2,
which enforces the bifurcation constraint. Since $\partitionGraph$ is
acyclic by construction, we do not require cycle elimination
constraints that are typically present in general formulations of
\acp{mcb}. Observing that $\forall e \in \edges_{ab} : 1 -
\branchvar{ab}=x_{e}$, \ie all edges in an arc need to have the same
state to satisfy space-time constraints, we derive the weights
$c_{ab}= - \sum_{e \in \edges_{ab}} c_e$. With a similar reasoning,
all vertices of a component $a$ need to be in the same
birth/termination state, $\forall v \in \vertices_{a} :
y_a^{+}=x_v^{+}$, hence we derive $c_{a}^+=\sum_{v \in \vertices_a}
c_v^+$ (and analogous for termination costs $c_{a}^-$). The derivation
is found in the supplement.

\paragraph{Matching-Based Algorithm for the \ac{mcb}.} We now show that
the \ac{mcb}~\eqref{eq:mcb-objective}-\eqref{eq:outgoing} can be
solved efficiently by reducing it to a set of \acp{mcbmp}.

To this end, observe that the graph $\partitionGraph =
  (\components, \arcs)$ is acyclic by construction, cf.\ Fig.\ \ref{fig:partition-graph}. Denote by $\partitionGraph_{t,t+1} = (\components_t \cup \components_{t+1}, \arcs_{t,t+1})$ the subgraph of $\partitionGraph$ that corresponds to the consecutive frames $t$ and $t+1$.

\begin{lemma} \label{lem:independent-mcb} For every $\partitionGraph =
  (\components, \arcs)$ arising from a fixed intra-frame decomposition, the solution of the \ac{mcb} on $\partitionGraph$ can be found by solving the \ac{mcb} for all
  $\partitionGraph_{t,t+1}$ individually.
\end{lemma}
\begin{proof}
The constraints \eqref{eq:incoming} only couple birth variables $y_{a}^+$ for $a \in \components_{t+1}$ with arc variables $y_{ba}$ where $ba \in \arcs_{t,t+1}$. Similarly, the constraints
  \eqref{eq:outgoing} only couple termination variables
  $y^-_a$ for $a \in \components_t$ with arc variables $y_{ab}$ where $ab \in \arcs_{t,t+1}$. Thus, the objective function and the constraints split into a set of \acp{mcb} corresponding to the subgraphs $\partitionGraph_{t,t+1}$ of $\partitionGraph$. Hence, solving $\vert
  \mathcal{T} \vert - 1$ many sub-\acp{mcb} individually gives the solution of the \ac{mcb} on $\partitionGraph$.
\end{proof}

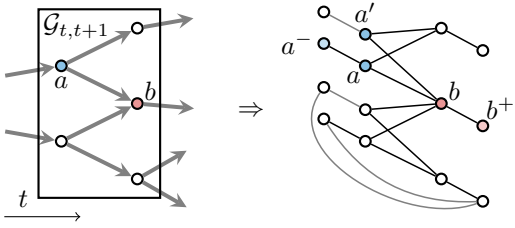
\begin{figure}
  \centering {

\definecolor{firstColor}{RGB}{20,133,204}
\definecolor{secondColor}{RGB}{212,42,42}

\tikzstyle{every node}=[circle, fill=white,
                        inner sep=0pt, minimum width=4pt]

\pgfdeclarelayer{bg}
\pgfsetlayers{bg,main}

\begin{tikzpicture}[thin]

  \draw[->] (0.25,0) -- node[above, yshift=0.1cm, xshift=-0.25cm] {$t$} (1.25, 0) ;
  \node[fill=none] at (3.5, 1.4) {$\Rightarrow$};

    \begin{scope}[thick, xshift=0cm]

      \node[draw] (a10) at (1,1) {};
      \node[draw,label={below}:$a$,fill=firstColor!50] (a11) at (1, 2) {};
      
      \node[draw] (a20) at (2., 2.5) {};
      \node[draw] (a21) at (2., .5) {};
      \node[draw,label={above right}:$b$,fill=secondColor!50] (a22) at (2.,1.5) {};
      
    \end{scope}

    \begin{scope}[gray, line width=1.5pt, >=stealth]
      \draw[<-] (a10) -- ++(170:0.75);
      \draw[<-] (a11) -- ++(190:0.75);
      
      \draw[->] (a11) -- (a20);
      \draw[->] (a11) -- (a22);
      \draw[->] (a10) -- (a22);
      \draw[->] (a10) -- (a21);

      \draw[->] (a20) -- ++(10:0.75);
      \draw[->] (a22) -- ++(-5:0.75);
      \draw[->] (a21) -- ++(30:0.75);
      \draw[->] (a21) -- ++(-30:0.75);
    \end{scope}

    \draw[thick] (0.7, 2.75) node[fill=none,label={above}:$\partitionGraph_{t,t+1}$, yshift=-0.8cm, xshift=0.5cm] {} rectangle (2.3, 0.25) ;

    \begin{scope}[thick, xshift=4cm]
      \node[draw] (a10) at (1,1) {};
      \node[draw,label={below left}:$a$,fill=firstColor!50] (a11) at (1, 2) {};

      \node[draw, above=0.25cm of a10] (a10p)  {};
      \node[draw, above=0.25cm of a11, label={above}:$a'$,fill=firstColor!50] (a11p) {};
      
      \node[draw, above left=0.25cm of a10, xshift=-0.25cm] (a10t)  {};
      \node[draw, above left=0.25cm of a10p, xshift=-0.25cm] (a10pt)  {};
      
      \node[draw, above left=0.25cm of a11, xshift=-0.25cm,  label={left}:$a^-$,fill=firstColor!25] (a11t)  {};
      \node[draw, above left=0.25cm of a11p, xshift=-0.25cm] (a11pt) {};

      \node[draw] (a20) at (2., 2.5) {};
      \node[draw] (a21) at (2., .5) {};
      \node[draw, label={above right}:$b$,fill=secondColor!50] (a22) at (2.,1.5) {};

      \node[draw, below right=0.25cm of a20, xshift=0.25cm] (a20b)  {};
      \node[draw, below right=0.25cm of a21, xshift=0.25cm] (a21b)  {};
      \node[draw, below right=0.25cm of a22, xshift=0.25cm, label={above right}:$b^+$,fill=secondColor!25] (a22b)  {};

    \end{scope}

    \begin{scope}[line width=0.2mm]
    \draw (a10) -- (a21);
    \draw (a10) -- (a22);
    \draw (a11) -- (a22);
    \draw (a11) -- (a20);

    \draw (a10p) -- (a21);
    \draw (a10p) -- (a22);
    \draw (a11p) -- (a22);
    \draw (a11p) -- (a20);

    \draw (a10) -- (a10t);
    \draw[color=gray] (a10p) -- (a10pt);
    \draw (a11) -- (a11t);
    \draw[color=gray] (a11p) -- (a11pt);

    \draw (a20) -- (a20b);
    \draw (a21) -- (a21b);
    \draw (a22) -- (a22b);

    \draw[gray] (a21b) to[bend left] (a10t);
    \draw[gray] (a21b) to[in=-125, out=-160] (a10pt);
  \end{scope}

\end{tikzpicture}
}

  \caption{Illustration of the constructed bipartite matching problem
    (\textbf{right}) for an \ac{mcb} in the subgraph of two
    consecutive frames $t,t+1$ (\textbf{left}). The matching problem
    graph consists of the original nodes and edges, duplicates $a'$
    for $a \in \components_t$, auxiliary termination nodes $a^-$ and
    auxiliary birth nodes $b^+$. Auxiliary edges which have zero cost
    by construction are gray. For simplicity, we illustrate only two
    edges between termination and birth nodes. Matched nodes
    correspond to active arcs in the original $\partitionGraph_{t,t+1}$.}
  \label{fig:matching}
\end{figure}

\begin{lemma} \label{lem:matching-mcb} An \ac{mcb} on
  $\partitionGraph_{t,t+1}$ can be transformed into an equivalent
  \acf{mcbmp}.
\end{lemma}
\begin{proof}
  For a given \ac{mcb} on $\partitionGraph_{t,t+1}$, we construct an
  \ac{mcbmp} as follows (illustrated in
  \figref{fig:matching}): 1) insert a duplicate $a'$ for each node $a
  \in \components_t$ and add an arc $a'b$ for each original arc $ab
  \in \arcs_{t,t+1}$ with identical cost $c_{a'b} = c_{ab}$. 2) For
  each node $a \in \components_t$, insert a node $a^-$ and an arc
  $aa^-$ with its cost being $c^-_a$, \ie the cost of terminating in
  $a$. Repeat this for all duplicate nodes $a'$ but set the according
  cost $c^-_{a'} = 0$. Similarly, add a node $b^+$ for each $b \in
  \components_{t+1}$ and an arc $b^+b$ with a cost of $c^+_b$. 3)
  Connect each pair of auxiliary nodes $b^+$ and $a^-$ (or $a'^-$)
  with an arc if $ab \in \arcs_{t,t+1}$ with a cost of 0. The
  resulting graph is clearly bipartite.
  
  Now, consider the \ac{mcbmp} on
  this graph: A match $(a,b)$ or $(a',b)$ corresponds to $y_{ab} = 1$, respectively $y_{a'b} = 1$,
  a match of $(a,a')$ to $y_a^-=1$ and vice versa for birth variables
  $y_b^+$. Exactly one incoming arc for each node of
  $\components_{t+1}$ or the link to its birth node $b^+$ is matched,
  satisfying~\eqref{eq:incoming}. In the same fashion, each $a \in
  \components_{t}$ is assigned to a node $b \in \components_{t+1}$ or
  its termination node $a^-$, satisfying the left hand side
  of~\eqref{eq:outgoing}. Assigning a duplicate node $a'$ to a node $b
  \in \components_{t+1}$ allows having bifurcations, \ie satisfies the
  right-hand side of~\eqref{eq:outgoing}, while its alternative
  choice, matching it to its zero-cost termination node has no effect
  on the cost. Finally, the zero-cost arcs between the auxiliary birth and
  termination nodes $a^-$ and $b^+$ are matched whenever a pair of $a$ or
  and $b$ is matched (due to lack of
  alternatives).
\end{proof}

\vspace{-1ex}

The \ac{mcbmp} can be solved in polynomial time by the hungarian
algorithm~\cite{kuhn-1955,munkres-1957}. Applying it to each of the
$\vert \mathcal{T} \vert - 1$ subgraphs of $G_{t,t+1}$ thus leaves us
with an efficient minimizer for the \ac{mcb}.

\vspace{-1ex}

\paragraph{Implementation of KLB.}
The algorithm maintains the weighted $\partitionGraph=(\components,
\arcs)$, the current objective in terms of each of the three parts of
\eqref{eq:composed-objective}, and solves the \ac{mcb} on
$\partitionGraph$ by the matching-based algorithm described in the
previous section. We initially solve the entire \ac{mcb}, but then,
within both methods that propose transformations,
$\mathrm{improveBipartition}$ and $\mathrm{splitPartition}$, we
exploit the locality of the introduced changes. By applying
Lemma~\ref{lem:independent-mcb}, we note that for a given
$\components$, modifying two of its cells $a$ and $b$ in frame $t$
will only affect arcs that go from $t-1$ to $t$ and from $t$ to
$t+1$. In other words, $\Delta f_{\mathrm{MCBP}}$ can be computed only
from the subproblems of $(t-1,t)$ and $(t,t+1)$.  In practice, we find
that the effect is often also spatially localized, hence we optionally
restrict ourselves to only updating the \ac{mcb} in a range of
$d_{\mathrm{MCBP}}$ (undirected) arc hops from $a$ and $b$, where the
modification occured. This $d_{\mathrm{MCBP}}$ parameter should be
explored and set depending on the instance, since choosing it too
small may result in misjudged moves and thus, in wrong incremental
changes to the current objective. Note, however, that feasibility is
still maintained in any case.  We handle this by solving the
\emph{entire} \ac{mcb} once at the end of every outer iteration. Doing
so ensures that the final objective is always correct and allows us to
detect choices of $d_{\mathrm{MCBP}}$ that are too small. Since we
observe that it takes relatively few outer iterations, we find the
overhead by these extra calls to be negligible.

To reduce the number of overall calculations in later iterations, we
mark components that have changed and then, in the next iteration,
attempt to improve only those pairs of components which involve at
least one changed component. To account for changes that affect moves
in previous or subsequent frames, we propagate these
$\mathrm{changed}$ flags to all potential parents or children of a
changed component.

\section{Improved Branch-and-Cut Algorithm} \label{sec:branch-and-cut}
Jug et al.\ propose to solve the \ac{mlt} with a branch-and-cut algorithm, for which they design separation procedures for inequalities \eqref{eq:space-cycle} -- \eqref{eq:morality}, \eqref{eq:birth} -- \eqref{eq:termination} and the bifurcation constraints. In the following, we propose several modifications of the optimization algorithm, which drastically improve its performance.

It is sufficient to consider only chordless cycles in \eqref{eq:space-cycle} and, furthermore, it is well-known that chordless cycle inequalities are facet-defining for multicut polytopes (cf.\ \cite{chopra-1993} and \cite{hornakova-2017}). This argument can be analogously transferred to inequalities \eqref{eq:space-time-cycle} and \eqref{eq:morality}.

Moreover, the inequalities of \eqref{eq:morality} where $\{v_t,w_t\} \in E_t$ is an edge of the hypothesis graph may be considerably strengthened by a less trivial, yet simple modification. Lemma \ref{lem:cycle-morality-ineq} shows that with both results combined, we can equivalently replace \eqref{eq:space-cycle} -- \eqref{eq:morality} by the set of tighter inequalities \eqref{eq:cycle} and \eqref{eq:morality-reduced}. Proofs are provided in the supplementary material. In relation to our improved version of the branch-and-cut algorithm, we refer to \eqref{eq:cycle} as \emph{cycle} and to \eqref{eq:morality-reduced} as \emph{morality} constraints.

\begin{lemma} \label{lem:cycle-morality-ineq}
For every hypothesis graph $G = (V,E)$ it holds that $x \in X'_G$ iff $x \in \{0,1\}^E$ and $x$ satisfies
\begin{align}
& \forall t \in \mathcal{T} \forall \{v,w\} \in E_t \cup E_{t,t+1} \nonumber \\ 
& \forall \text{ chordless $vw$-paths } P \text{ in } G_t^+: \nonumber \\
& \qquad x_{vw} \leq \sum_{e \in P} x_e \label{eq:cycle} \\
& \forall t \in \mathcal{T} \forall v',w' \in V_t \text{ such that } \{v',w'\} \notin E_t \nonumber \\ 
& \forall v'w'\text{-cuts } S \text{ in } G_t \forall \text{ chordless $v'w'$-paths } P \text{ in } G_t^+: \nonumber \\
& \qquad 1 - \sum_{e \in S} (1 - x_e) \leq \sum_{e \in P} x_e \label{eq:morality-reduced}
\end{align}
\end{lemma}

\begin{remark}
Suppose we introduce for every pair of non-neighboring nodes $v',w' \in V_t$ a variable $x_{v'w'}$ indicating whether $v'$ and $w'$ belong to the same cell ($x_{v'w'} = 0$) or not ($x_{v'w'} = 1$). Then any inequality of \eqref{eq:morality-reduced} is exactly the combination of a \emph{cut} inequality $1 - x_{v'w'} \leq \sum_{e \in S} (1 - x_e)$ and a \emph{path} inequality $x_{v'w'} \leq \sum_{e \in P} x_e$ in the sense of \emph{lifted multicuts} \cite{hornakova-2017}. For neighboring nodes $v,w \in V_t$, i.e.\ $\{v,w\} \in E_t$, we have the variable $x_{vw}$ at hand and can thus omit the cut part of the morality constraint, as the lemma shows.
\end{remark}

\vspace{-2ex}

\paragraph{Termination and Birth Constraints.}

We further suggest a strengthening of the birth and termination constraints in the \ac{mlt}. To this end, for any $v \in V_{t+1}$ let $V_t(v) = \{ u \in V_t \mid \{u,v\} \in E_{t,t+1}\}$ be the set of neighboring nodes in frame $t$. Further, we denote by $E\big(V_t(v),V_{t+1}\setminus \{v\}\big)$ the set of inter frame edges that connect some node $u_t \in V_t(v)$ with some node $u_{t+1} \in V_{t+1}$ different from $v$.

\begin{lemma} \label{lem:birth-termination}
For every hypothesis graph $G = (V,E)$, the vectors $x \in X'_G, \; x^+, x^- \in \{0,1\}^V$ satisfy inequalities \eqref{eq:birth} iff the following inequalities hold:
\begin{align}
\forall t \in \mathcal{T} \forall v \in V_{t+1} \forall S \in V_tv\text{-cuts}(G_t^+) : \nonumber \\
1 - x_v^+ \leq \sum_{e \in S \setminus E(V_t(v),V_{t+1}\setminus \{v\})} (1 - x_e). \label{eq:birth-reduced}
\end{align}
Similarly, $x \in X'_G, \; x^+, x^- \in \{0,1\}^V$ satisfy \eqref{eq:termination} iff
\begin{align}
\forall t \in \mathcal{T} \forall v \in V_{t} \forall S \in vV_{t+1}\text{-cuts}(G_t^+) : \nonumber \\
1 - x_v^- \leq \sum_{e \in S \setminus E(V_t \setminus \{v\},V_{t+1}(v))} (1 - x_e) \label{eq:termination-reduced}
\end{align}
hold true.
\end{lemma}

\paragraph{Additional Odd Wheel Constraints.}

A \emph{wheel} $W = (V(W), E(W))$  is a graph that consists of a cycle and a dedicated center node $w \in V(W)$ which is connected by an edge to every node in the cycle. Let $E_C$ denote the edges of $W$ in the cycle and $E_w$ the remaining center edges. With a wheel subgraph $W = (V(W), E(W))$ of a graph $G$ we may associate an inequality
\begin{align}
\sum_{e \in E_C} x_e - \sum_{e \in E_w} x_e \leq \left \lfloor \frac{\lvert V(W) \rvert - 1}{2} \right \rfloor, \label{eq:wheel}
\end{align}
which is valid for multicut polytopes \cite{chopra-1993}. A wheel is called \emph{odd} if $\lvert V(W) \rvert - 1$ is odd. It is known that wheel inequalities are facet-defining for multicut polytopes iff the associated wheel is odd \cite{chopra-1993}.

We propose to add additional odd wheel inequalities to the \ac{mlt} in order to strengthen the corresponding LP relaxation. More precisely, we consider only wheels $W = (V(W),E(W)) \subset G$ such that $w \in V_{t+1}$ and $v \in V_t$ for all $v \in V(W) \setminus w$ and some $t \in \mathcal{T}$. This structure guarantees that for any $x \in X'_G$, the restriction $x_{E(W)}$ is the incidence vector of a multicut of $W$. Therefore, \eqref{eq:wheel} holds with respect to $x$.

\begin{figure}[t]
  \centering
  \includegraphics[width=\linewidth]{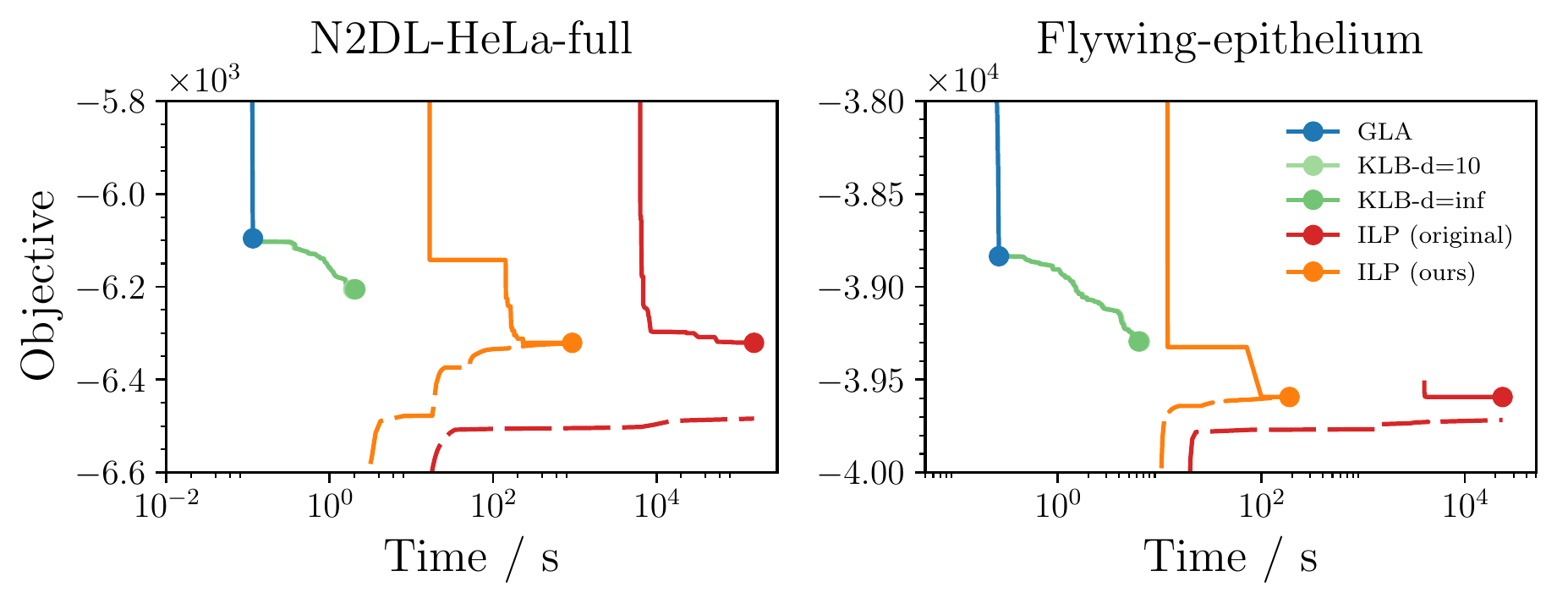}

  \caption{Comparison of algorithms for the \ac{mlt} in terms of
    runtime, objective (solid) and bounds (dashed) on the large
    instances of \cite{jug-2016}. Our heuristics are able to determine
    feasible solutions quickly, while our branch-and-cut algorithm
    (ILP ours) converges to the optimal solution in up to one
    hundredth of the time of the original branch-and-cut algorithm
    (ILP original) and provides tight bounds in both cases.  On these
    instances, KLB exhibits no significant runtime difference between
    the two choices of $d_{\mathrm{MCBP}}$.}
  \label{fig:runtime}
\end{figure}

\vspace{-1ex}

\paragraph{Implementation.}
For a subset of the constraints, we use the commercial branch-and-cut solver Gurobi (7.0) \cite{gurobi} to solve the LP relaxation and find integer feasible solutions. Whenever Gurobi finds an integer feasible solution $x$, we check whether $x \in X_G$ and all birth and termination constraints are satisfied. If not, then we provide Gurobi with an additional batch of violated inequalities from \eqref{eq:cycle} -- \eqref{eq:termination-reduced} as well as violated bifurcation constraints and repeat. To this end, we adapt the separation procedures of \cite{jug-2016} to account for our improvements in a straight-forward manner. We further add odd wheel inequalities for wheels with 3 outer nodes as described above (so-called \emph{3-wheels}) to the starting LP relaxation.

For every integer feasible solution that Gurobi finds, we fix the connected components of the intra-frame segmentation and solve the remaining \ac{mcb}. This allows for the early extraction of feasible lineage forests from the ILP.

\section{Experiments \& Results}

\paragraph{Instances and Setup.}
We evaluate our algorithms on the two large instances
of~\cite{jug-2016}: \emph{Flywing-epithelium} and
\emph{N2DL-HeLa-full}. The hypothesis graph of the former instance
consists of 5026 nodes and 19011 edges, while the latter consists of
10882 nodes and 19807 edges.  In addition to this, we report
experiments on two more sequences of a flywing epithelium time-lapse
microscopy with a wider field of view. Their hypothesis
graphs consist of 10641 nodes and 42236 edges, respectively 76747
edges. We denote the data sets with \emph{Flywing-wide I} and
\emph{II}. These instances are preprocessed with the same pipeline as
\emph{Flywing-epithelium}. For details on the preprocessing, we refer
to~\cite{jug-2016}.

Our choice of birth and termination costs follows \cite{jug-2016}, \ie
we set $\cost{}^+ = \cost{}^-=5$ for all instances.  We initialize the
KLB heuristic with the solution of GLA to decrease the number of outer
iterations. We benchmark two versions of KLB: The first one is denoted
with KLB-d=inf and solves the \ac{mcb} within the (reachable) subgraph
of $t \pm 1$, while the second, KLB-d=10, additionally exploits
spatial locality, \ie it uses $d_{\mathrm{MCBP}}=10$.

\begin{figure}
  \centering
  \includegraphics[width=\linewidth]{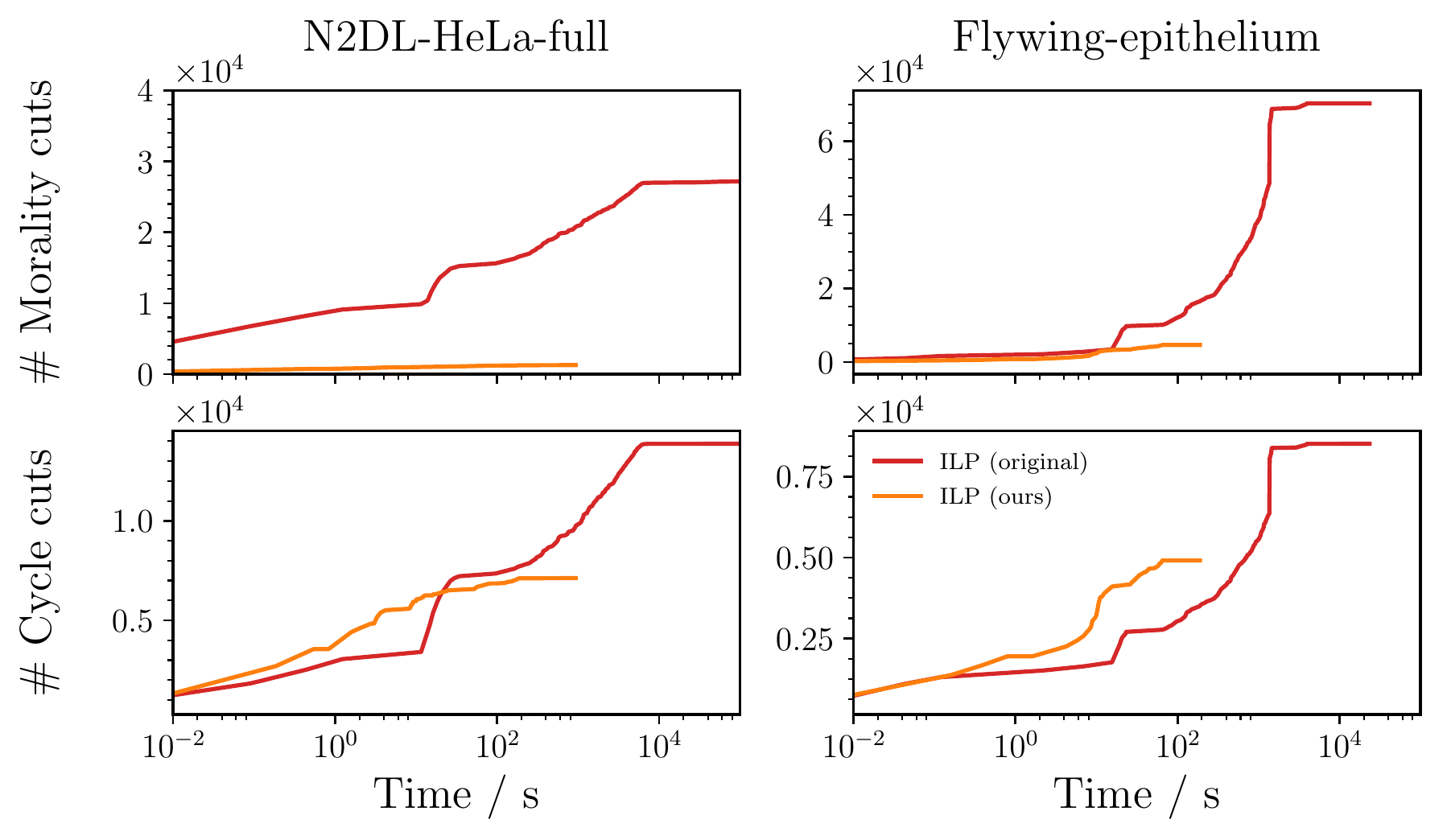}
  \caption{Number of morality cuts (\textbf{top}), \ie
    \eqref{eq:morality} or \eqref{eq:morality-reduced}, and cycle cuts
    (\textbf{bottom}), \ie \eqref{eq:space-cycle} and
    \eqref{eq:space-time-cycle} or \eqref{eq:cycle}, separated in the
    different branch-and-cut algorithms. We observe that our
    branch-and-cut algorithm requires considerably fewer morality
    cuts, while the number of cycle cuts (including both space-cycles
    and space-time-cycles) is in the same order of magnitude.}
  \label{fig:number-of-cuts}
\end{figure}

\begin{table*}
  \centering
  \caption{Detailed quantitative comparison of algorithms for the \ac{mlt}. BestGap is calculated using the tightest bound of any algorithm, while Gap is based on the bound established by each particular algorithm. KLB-d=inf solves the \ac{mcb} in the entire reachable subgraph of $\{ t-1, t, t+1 \}$, while KLB-d=10 additionally uses spatial locality with $d_{\mathrm{MCBP}}=10$.}
  \label{table:runtime}
  \begin{footnotesize}
    \begin{tabularx}{\textwidth}{X|rrrrr|rrrrr}
      \toprule
      & \multicolumn{5}{c|}{ Flywing-epithelium } & \multicolumn{5}{c}{N2DL-HeLa-full} \\
      Method                          &   Time / s &    objBest &   objBound &        Gap &    BestGap &   Time / s &    objBest &   objBound &        Gap &    BestGap \\
      \midrule
      GLA                             &       0.26 &  -38835.90 &            &            &     0.0195 &       0.12 &   -6095.85 &            &            &     0.0369 \\
      KLB-d=10                        &       6.42 &  -39294.65 &            &            &     0.0076 &       1.95 &   -6205.54 &            &            &     0.0186 \\
      KLB-d=inf                       &       6.24 &  -39294.65 &            &            &     0.0076 &       2.06 &   -6205.54 &            &            &     0.0186 \\
      ILP (ours)                      &     189.41 &  -39593.90 &  -39593.90 &     0.0000 &     0.0000 &     931.07 &   -6320.81 &   -6320.81 &     0.0000 &     0.0000 \\
      ILP (original)~\cite{jug-2016}  &   23460.80 &  -39593.90 &  -39717.80 &     0.0031 &     0.0000 &  156542.00 &   -6320.81 &   -6484.02 &     0.0258 &     0.0000 \\
      \bottomrule
    \end{tabularx}
  \end{footnotesize}
\end{table*}

\vspace{-1em}

\paragraph{Convergence Analysis.}
The convergence of our algorithms in comparison to the branch-and-cut
algorithm of~\cite{jug-2016} is reported in \figref{fig:runtime} and
\tableref{table:runtime}. We find that GLA is the fastest in all
instances, but only reaches a local optimum with a gap of about
$1.95\,\%$ and $3.69\,\%$, respectively.  This solution is improved by
KLB in terms of objective, up to a gap of $0.76\,\%$ and
$1.86\,\%$. Both variants of KLB obtain the same solution in terms of
cut-edge labeling and show no considerable runtime difference. We find
that KLB spends most of the time in the first outer iteration, where
it has to check a large number of bipartitions that do not improve and
will therefore not be considered in the next iteration. Our KLB
implementation could potentially be sped up by updating components (of
disjoint $\partitionGraph_{t-1:t+1}$) in
parallel.

The improved branch-and-cut algorithm retrieves feasible solutions
considerably faster and provides tighter bounds than the algorithm
of~\cite{jug-2016}. The instances \textit{Flywing-epithelium} and
\textit{N2DL-HeLa} are solved to optimality in less than $200\,s$,
respectively $1000\,s$, while the original algorithm did not find any
feasible solutions in that time. As shown in
\figref{fig:number-of-cuts}, we observe that our modifications of the
branch-and-cut algorithm greatly reduce the number of morality cuts.

On the larger instances \emph{Flywing-wide I} and \emph{II}, we
present our results in~\figref{fig:runtime-wide}. We are able to
determine the maximal optimality gaps for GLA to be $2.9\,\%$
(\emph{I}) and $2.1\,\%$ (\emph{II}), and $1.3\,\%$ (\emph{I}) and
$0.95\,\%$ (\emph{II}) for KLB. Again, both variants of KLB obtain
identical solutions. Here, exploiting spatial locality helps:
KLB-d=inf runs in $477\,s$ (\emph{I}) and $9129\,s$ (\emph{II}), while
KLB-d=10 reduces this to $104\,s$ and $3359\,s$, respectively. The
particular choice of $d_{\mathrm{MCBP}}=10$ was found to be stable in
both cases. More extensive results with varying $d_{\mathrm{MCBP}}$
can be found in the supplement.

\begin{figure}
  \centering
  \includegraphics[width=\linewidth]{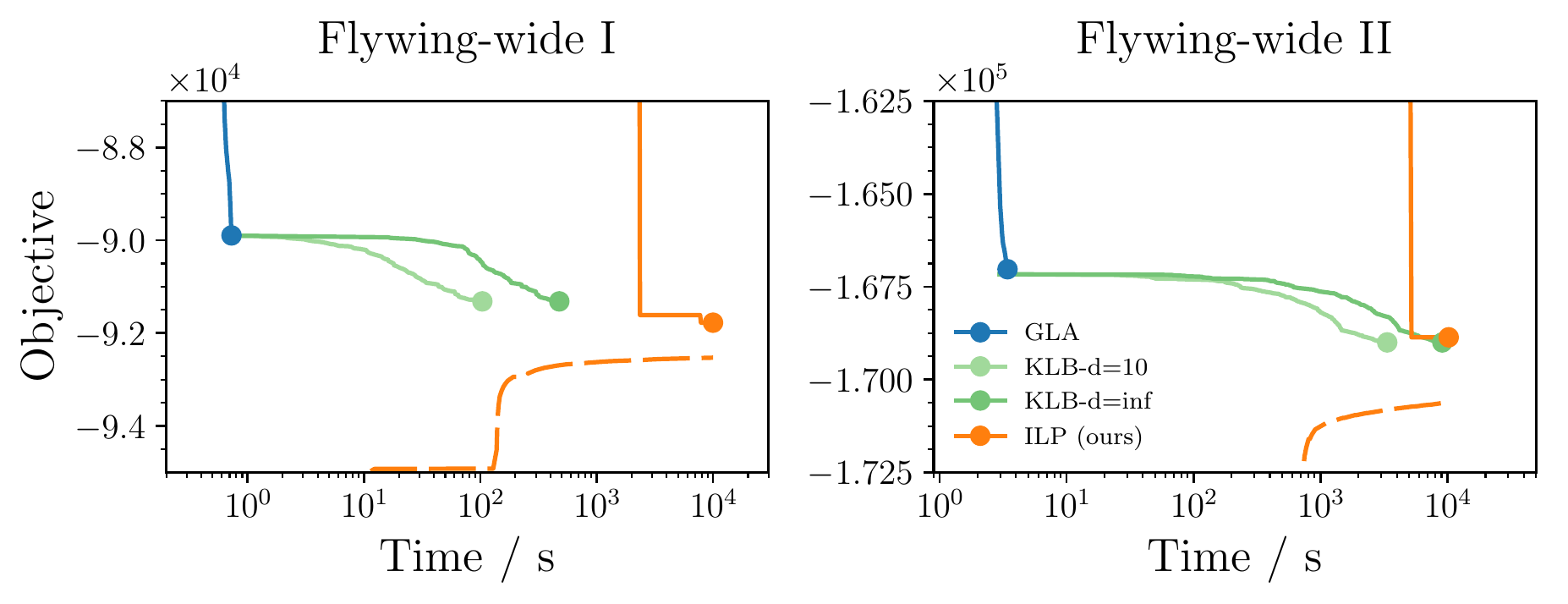}

  \caption{Results on the more extensive instances \emph{Flywing-wide
      I} and \emph{II}. Our branch-and-cut algorithm with 3-wheel
    constraints provides slightly tighter bounds, with which we
    determine the gaps for GLA to be $2.9\,\%$ (\emph{I}) and
    $2.1\,\%$ (\emph{II}), and $1.3\,\%$ (\emph{I}) and $0.95\,\%$
    (\emph{II}) for KLB. Exploiting spatial locality when re-solving
    the \acp{mcb} considerably reduces runtime of KLB.}
  \label{fig:runtime-wide}
\end{figure}

\begin{table}
  \caption{Comparison of the similarity to ground truth of segmentation (SEG) and traced lineage forest (TRA) on \emph{Flywing-epithelium}.
    ILP denotes the result of the branch-and-cut algorithm, while PA~\cite{aigouy-2010} is a common tracking tool used by biologists.}
  \label{table:scores}
  \begin{center}
    \begin{small}
      \begin{tabular}{lll}
        \toprule
        Algorithm & SEG & TRA \\
        \midrule
        GLA  & 0.9363 & 0.9640 \\
        KLB  & 0.9485 & 0.9721 \\
        ILP & 0.9722 & 0.9813 \\
        PA (auto) & 0.7980 & 0.9206 \\
        \bottomrule
      \end{tabular}
    \end{small}
  \end{center}
  \vspace{-1.5em}
\end{table}

\vspace{-1em}
\paragraph{Solution Quality.}
We compare the solution quality of our two heuristics by segmentation
(SEG) and tracking (TRA) metrics as used in~\cite{maska-2014} for \emph{Flywing-epithelium}. The
results are reported in \tableref{table:scores}. We observe that KLB
improves the scores of GLA slightly (up to an additional $1.2\,\%$ and
$0.81\,\%$ for SEG and TRA, respectively). The optimal ILP
solutions achieve slightly better scores in both measures
than the heuristics.  All presented algorithms outperform the
baseline, the \emph{packing analyzer}~\cite{aigouy-2010},
whose scores were originally reported in~\cite{jug-2016}.

\section{Conclusion}

We have introduced local search algorithms for the recently introduced
\ac{mlt}~\cite{jug-2016}, a mathematical framework for cell lineage
reconstruction, which treats both subproblems, image decomposition and
tracking, jointly. We propose two efficient heuristics for the
\ac{mlt}: a fast agglomerative procedure called GLA that constructs a
feasible lineage bottom-up, and a variant of the KL-algorithm which
attempts to improve a given lineage by switching nodes between
components, merging or splitting them. The latter algorithm repeatedly
solves a \ac{mcb} conditioned on fixed partitions. We show that this
subproblem can be solved as a minimum cost bipartite matching problem,
which is of independent interest.  Furthermore, we improve the
branch-and-cut algorithm of~\cite{jug-2016} by separating tighter
cutting planes and employing our result about the \ac{mcb}
subproblem. Our branch-and-cut algorithm solves previous instances
quickly to optimality.  For both the previous and larger instances,
our heuristics efficiently find high quality solutions. This
demonstrates empirically that our methods alleviate runtime issues
with \ac{mlt} instances and makes \emph{moral lineage tracing}
applicable in practice (\eg in~\cite{rempfler-2017-miccai}).

{\small \bibliographystyle{ieee} \bibliography{manuscript-lib} }

\end{document}